\documentclass[11pt]{article}
\usepackage{url}
\usepackage[utf8]{inputenc}
\usepackage{amsmath,color,amssymb,amsthm}
\usepackage[letterpaper,margin=1.0in]{geometry}
\usepackage[ruled,linesnumbered]{algorithm2e}
\usepackage{array}
\usepackage{multirow}
\usepackage[hidelinks]{hyperref}
\usepackage{booktabs}
\usepackage{epsfig}
\usepackage{epstopdf}
\usepackage{algpseudocode}
\usepackage{graphicx}
\usepackage{color}
\usepackage{authblk}

\definecolor{officegreen}{rgb}{0.0, 0.5, 0.0}

\def\R{\mathbb{R}}
\def\est{\textnormal{est}}
\def\obs{\textnormal{obs}}
\def\W{\mathrm{W}}
\DeclareMathOperator*{\diag}{\mathrm{diag}}
\usepackage{bm}
\newcommand{\bracket}[1]{\left\langle#1\right\rangle}

\newtheorem{theorem}{Theorem}[section]

\title{Recovering Wasserstein Distance Matrices from Few Measurements}
\author[1]{Muhammad Rana}
\author[3]{Abiy Tasissa}
\author[4]{HanQin Cai}
\author[1]{Yakov Gavriyelov}
\author[1,2]{Keaton Hamm}
\affil[1]{Department of Mathematics, University of Texas at Arlington}
\affil[2]{Division of Data Science, University of Texas at Arlington}
\affil[3]{Department of Mathematics, Tufts University, Medford, MA 02155, USA.}
\affil[4]{Department of Statistics and Data Science and Department of Computer Science\\University of Central Florida}
\date{}

\begin{document}
\allowdisplaybreaks
%
\maketitle
\begin{abstract}
This paper proposes two algorithms for estimating square Wasserstein distance matrices from a small number of entries. These matrices are used to compute manifold learning embeddings like multidimensional scaling (MDS) or Isomap, but contrary to Euclidean distance matrices, are extremely costly to compute. We analyze matrix completion from upper triangular samples and Nystr\"{o}m completion in which $\mathcal{O}(d\log(d))$ columns of the distance matrices are computed where $d$ is the desired embedding dimension, prove stability of MDS under Nystr\"{o}m completion, and show that it can outperform matrix completion for a fixed budget of sample distances. Finally, we show that classification of the OrganCMNIST dataset from the MedMNIST benchmark is stable on data embedded from the Nystr\"{o}m estimation of the distance matrix even when only 10\% of the columns are computed.
\end{abstract}
%
%
\section{Introduction}
\label{sec:intro}

The bare representation of data in many applications is high-dimensional \cite{zheng2009statistical}. 
In today’s era of big data, this high dimensionality poses significant challenges and necessitates specialized methods. Manifold learning refers to a class of techniques that compress data representations into lower-dimensional forms, making them more suitable for analysis and computation. 
A widely used linear technique is principal component analysis (PCA). 
However, linear approaches can be restrictive when the data lie on or near a nonlinear manifold. To address this, nonlinear manifold learning techniques aim to obtain compact representations while preserving local geometric structures \cite{meilua2024manifold}.
We note that, in the standard way they are used, all of the aforementioned methods typically treat data as vectors in a Euclidean space.

While many applications yield vector data, there are a number of domains that produce data that may be best modeled as distributions or measures. For example, gene expression and natural language processing data can be represented as distributions over gene networks or dictionaries, respectively, and image data can be considered as discrete distributions over a 2-dimensional pixel grid. Recent works have used this notion to improve understanding of data geometry and performance on tasks such as classification or clustering \cite{mathews18,kusner2015word,wang2010optimal}. One of the advantages of treating images, for example, as measures, is that spatial relationships are not obscured as they are under vectorization. Moreover, the natural notion of a metric structure on the space of measures is that of the Wasserstein metric from optimal transport (OT) \cite{peyre2019computational}. 


Several works have explored manifold learning using probability measures and the Wasserstein metric, e.g., \cite{wang2010optimal,hamm2023wassmap}. This approach has shown improvement in task performance and data modeling, but it comes with the significant drawback of slow computation time. In particular, computing one Wasserstein distance exactly for discrete measures supported on $m$ points costs $\Omega(m^3\log(m))$ flops. To form neighbor embeddings for $N$ data measures requires computing the square distance matrix $D_{ij}=\W_2(\mu_i,\mu_j)^2$, carrying total complexity $\mathcal{O}(N^2m^3\log(m))$.

Mitigating this high computational cost can be approached in two ways: 1) speeding up a single Wasserstein distance computation, and 2) computing fewer entries of $D$ to form an approximation of it. Many methods from computational OT explore the first approach, e.g., through entropic regularization, multiscale methods, or Nystr\"{o}m approximations to the cost matrix in the discrete formulation of the Kantorovich problem (see \cite{peyre2019computational} for a survey). In contrast, relatively few works have explored the problem of estimating an entire Wasserstein distance matrix directly from samples.

The focus of this paper is to propose algorithms for estimating square Wasserstein distance matrices for use in neighbor embeddings by computing as few entries as possible. Our approach is based on exploiting the approximate low-rank structure of Wasserstein squared distances: if the matrix 
$D$ is approximately rank $r$, we propose two algorithms based on matrix completion and the Nystr\"{o}m method which allow us to estimate $D$ from $\mathcal{O}(r\log(r))$ fully observed columns and maintain stability of the MDS embedding. 


\section{Technical background}
\subsection{Optimal transport and MDS}
The Kantorovich formulation of the quadratic Wasserstein metric on discrete probability measures $\mu=\sum_{i=1}^n\alpha_i\delta_{x_i}$ and $\nu=\sum_{i=1}^m\beta_i\delta_{y_i}$ is
\begin{equation}\label{eq:W2-metric}
    \W_2(\mu,\nu) := \min_{P\in\Gamma(\mu,\nu)} \bracket{C,P}
\end{equation}
where $\Gamma(\mu,\nu)=\{P\in\R^{m\times n}_+: P\mathbf{1}_m = \alpha, P^\top\mathbf{1}_n=\beta\}$ are couplings with marginals $\mu$ and $\nu$, and $C_{ij} = \|x_i-y_i\|^2$ \cite{peyre2019computational}. This problem computes an optimal matching of mass between two discrete measures. Given measures $\{\mu_i\}_{i=1}^N\subset\W_2(\R^n)$, the Wassmap algorithm \cite{wang2010optimal,hamm2023wassmap} computes $D_{ij} = \W_2(\mu_i,\mu_j)^2$, a double centering of it $B = -\frac12HDH$ where $H = I-\frac1N\mathbf{1}\mathbf{1}^\top$, with truncated SVD $B_d = V_d\Sigma_dV_d^\top$. The embedding vectors in $\R^d$ are then rows of $V_d\Sigma_d$. 

\subsection{Matrix completion}
Given a low-rank matrix with missing entries, the problem of estimating the unknown entries is known as the matrix completion problem. This problem has been studied extensively due to its broad applications, most notably in recommendation systems \cite{ramlatchan2018survey}. A naive formulation seeks to minimize the rank of the matrix subject to the constraints imposed by the observed entries; however, rank minimization is computationally hard in general \cite{meka2008rank}. A common convex relaxation is to minimize the nuclear norm, which serves as a surrogate for rank. In this convex setting, seminal theoretical results show that, under random sampling, an $N \times N$ rank-$r$ matrix can be exactly recovered from $\mathcal{O}(\nu r N  \log N)$ observed entries, where $\nu$ is an incoherence parameter capturing the conditioning of the underlying matrix \cite{candes2006robust}. Given the computational challenges of convex methods, non-convex algorithms have also been explored, and similar recovery guarantees have been established in this setting \cite{chi2019nonconvex}. 
In this paper, we focus on the completion of distance matrices. In the case of Euclidean distance matrices, it can be shown that the squared distance matrix has rank at most $d+2$, where $d$ is the embedding dimension of the underlying points \cite{liberti2014euclidean}. This structural property enables the use of low-rank optimization techniques for Euclidean distance matrix completion \cite{moreira2018novel,tasissa2018exact}.

\subsection{Nystr\"{o}m method}
Consider a symmetric matrix in block form $D=\begin{pmatrix}A & B\\
B^{\top} & C\end{pmatrix}$, where $A\in \R^{m\times m}$ and $B\in \R^{m\times n}$. We are interested in the regime $m\ll n$, where the aim is to estimate the block $C$ using only the smaller blocks $A$ and $B$. Suppose that  $\mathrm{rank}(A)=\mathrm{rank}(D)$. Then, one can show that the columns of the submatrix $\begin{pmatrix}B\\C   \end{pmatrix}$ lie in the span of the columns of $\begin{pmatrix}A\\B^\top   \end{pmatrix}$. Specifically, there exists an $X\in \R^{m\times n}$ such that $B=AX$ and $C=B^{\top} X$. If $A$ is invertible, then $X=A^{-1}B$, which yields the exact relation $C=B^\top A^{-1}B$. In the more general case where $A$ is not invertible, we may instead use the Moore--Penrose pseudoinverse to obtain  $X=A^{\dagger}B$ leading to $C=B^{\top}A^{\dagger}B$. This procedure is known as the Nystr\"{o}m method \cite{williams2000using}. 

While it is not a completion algorithm per se, the Nystr\"{o}m method can be used to complete a distance matrix. In particular, if only the columns containing $A$ and $B^\top$ are computed, then the entries of $C$ can be recovered, or in the general case approximated, by $B^{\top}A^\dagger B$.

\vspace{0.08in}
\noindent \textbf{Notations} For a distance matrix $D$, owing to symmetry and zero diagonal entries, we only consider samples from the strictly upper-triangular part. We denote by $\Omega \subset\{(i,j):i\in[N],j>i\}$ the set of sampled index pairs, with cardinality $|\Omega|=m$. Given a matrix $D$ and an index set $I\subset[N] := \{1,\dots,N\}$, we write $D(:,I)$ for the submatrix of 
$D$ consisting of the columns indexed by $I$. Analogously, $D(I,:)$ denotes the submatrix formed by the rows indexed by $I$. $P_\Omega(D)_{ij} = D_{ij}$ if $(i,j)\in\Omega$ and is 0 otherwise. If $\mathrm{rank}(D)=r$ and $D = V_r\Sigma_rV_r^\top$, then the incoherence $\nu$ of $D$ is $\sqrt{n/r}\max_i\||V_r(i,:)\|$. 

\section{Algorithms}

Here we describe two algorithms for estimating the Wasserstein distance matrix to compute embeddings of images. Algorithm \ref{alg:completion}, $\W_2$\texttt{-MC}, randomly computes a given number of entries of the distance matrix, and then uses a matrix completion algorithm to estimate the full squared distance matrix. We leave the completion algorithm as a generic plug-in for existing algorithms, and consider one designed for distance matrices in our experiments. We sample from the strict upper triangle, include the diagonal, and reflect the entries across the diagonal to obtain the observed distance matrix.
We apply the same procedure after completion to ensure that the resulting matrix retains the properties of a distance matrix.

\begin{algorithm}[h!]
\SetKwInOut{Input}{Input}

\textbf{Input:} Data measures $\{\mu_i\}_{i=1}^N$, sample size $m$, matrix completion subroutine \texttt{MC}

 Randomly sample $m$ entries $\Omega$ 
 
 Compute observed distance matrix $D_\obs = P_\Omega(D)$, where $D_{ij}=\W_2(\mu_i,\mu_j)^2$

 $\diag(D_\obs) = 0$ \textcolor{officegreen}{\algorithmiccomment{enforce 0 diagonal}}

 $D_\obs = D_\obs+D_\obs^\top$ \textcolor{officegreen}{\algorithmiccomment{reflect across diagonal}}
  
 $D_\est = \texttt{MC}(D_\obs)$


 $D_\est = \frac12(D_\est+D_\est^\top)$ \textcolor{officegreen}{\algorithmiccomment{enforce symmetry}}
 
\textbf{Output:} Estimated distance matrix $D_\est$ 
 \caption{$\W_2$\texttt{-MC}}\label{alg:completion}
\end{algorithm}

Algorithm \ref{alg:Nystrom}, $\W_2$\texttt{-Nystr\"{o}m}, uses the Nystr\"{o}m method to complete the distance matrix. In particular, we compute only a set of columns of the distance matrix $C_\obs = D(:,I)$, and estimate $D_\est = C_\obs UC_\obs^\top$ where $U = C_\obs(I,:)=D(I,I)$. As in $\W_2$\texttt{-MC}, we also enforce the distance matrix structure by enforcing the 0 diagonal. 

\begin{algorithm}[h!]

\textbf{Input:}\,Data measures$\{\mu_i\}_{i=1}^N$, number of columns $c$ 

 Randomly sample $c$ columns $I\subset[N]$ 
 
 Compute columns of distance matrix $C_\obs = \begin{cases} \W_2(\mu_i,\mu_j)^2, & i\in [N], j\in I\\ 0, & \textnormal{otherwise}.\end{cases}$ 

 
 $D_\est = C_\obs UC_\obs^\top$, where $U = C_\obs(I,:)$

$\diag(D_\est) = 0$ \textcolor{officegreen}{\algorithmiccomment{enforce 0 diagonal}}

 
\textbf{Output:} Estimated distance matrix $D_\est$
 \caption{$\W_2$\texttt{-Nystr\"{o}m}}\label{alg:Nystrom}
\end{algorithm}


Note that for large $N$, the computational cost of both algorithms is dominated by the cost of computing the $\W_2$ distances, and the cost of the matrix completion subroutine \texttt{MC} is often negligible by comparison. Thus, algorithms that limit sampling complexity of the matrix are more desirable. For this reason, we do not truncate the rank of the Nystr\"{o}m approximation to $D$ in $\W_2$\texttt{-Nystr\"{o}m}, which is usually done to save computation time.

The Nystr\"{o}m method is well-studied, and using some existing bounds, we are able to prove a stability result for MDS embeddings on Wasserstein distance matrices. We assume that the data measures in $\W_2(\R^n)$ are isometric to a set of Euclidean vectors $\{y_i\}_{i=1}^N\subset\R^d$. Examples of such sets have been studied and include common image articulations like translations, scalings, and shears of a fixed base measure \cite{khurana2022supervised,hamm2023wassmap}. We assume noise (or error) in the computation of the $\W_2$ distances, and consider how the MDS embedding $\{z_i\}_{i=1}^N\subset\R^d$ coming from $D_\est$ of $\W_2$\texttt{-Nystr\"{o}m} compares to $\{y_i\}$. MDS embeddings are centered and equivalent up to orthogonal transformation, so we measure error by a Procrustes distance, defined as $\min_{Q \in O(d)}\, \|Z-QY\|_2$, where $O(d)$ is the set of orthogonal $d\times d$ matrices. Note that $Z$ and $Y$ are matrices whose columns are the embedded vectors.

\begin{theorem}\label{thm:nystrom_approx}
    Suppose $\{\mu_i\}_{i=1}^N\subset \W_2(\R^n)$, and $D_{ij}=\W_2(\mu_i,\mu_j)^2$ has incoherence $\nu$. Suppose there exist $\{y_i\}_{i=1}^n$ $\subset\R^d$ such that $D_{ij}=\|y_i-y_j\|^2$. Let $\lambda>0$, $\delta\in(0,1)$, and compute $m\gtrsim \lambda\nu^2(d+2)\max\{\log(d+2),\log(1/\delta)\}$ columns uniformly randomly in $\W_2$\texttt{-Nystr\"{o}m} with error $C_{\obs_{i,j}} = \W_2(\mu_i,\mu_j)^2+E_{ij}$ such that $\|E\|_2 \leq \sigma_{d+2}(D)/3$. If $\|Y^\dagger\|_2\|E\|_2^{1/2}\leq 1/\sqrt{2}$, then $D_\est=C_\obs UC_\obs^\top$ yields an MDS embedding $\{z_i\}_{i=1}^N\subset\R^d$ such that
    \begin{multline*}\min_{Q\in O(d)}\|Z-QY\|_2 \leq (1+\sqrt{2})\|Y^\dagger\|_2\|E\|_2\;\times\\ \left(5\frac{N}{(d+2)(1-\delta)}+3\sqrt{\frac{N}{(d+2)(1-\delta)}}+2\right)\end{multline*}
    with probability at least $1-\max\{re^{-\lambda(\delta+(1-\delta)\log(1-\delta))},\delta\}$.
\end{theorem}
\begin{proof}
A simple modification of the proof of \cite[Theorem 3.3]{cloninger2023linearized} implies the upper bound with $(1+\sqrt{2})\|Y^\dagger\|_2\|D-D_\est\|_2$. Then \cite[Theorem 3]{talwalkar2014matrix} implies that the columns selected would yield an exact Nystr\"{o}m decomposition of $D$ without the presence of noise. Thus, the hypotheses of \cite[Remark 4.7]{hamm2021perturbations} are satisfied, and yield that $\|D-D_\est\|_2 \leq 5\|V_{d+2}(I,:)^\dagger\|_2^2+3\|V_{d+2}(I,:)^\dagger\|_2+2$ where $I$ is the set of columns selected by Algorithm \ref{alg:Nystrom} and $V_{d+2}$ are the first $d+2$ left singular vectors of $D$. Finally, \cite[Lemma 3.4]{tropp2011improved} shows that with the given probability, $\|V_{d+2}(I,:)^\dagger\|_2 \leq \sqrt{N/((d+2)(1-\delta))}$ for all $\delta\in(0,1).$
\end{proof}

A simple corollary is that if we compute $\mathcal{O}((d+2)\log(d+2))$ columns of $D$ without error, then the MDS embedding is exact up to orthogonal transformation. Exact MDS embeddings of $\W_2$ data were studied in \cite{hamm2023wassmap}.

\section{Numerical Experiments}

\subsection{Fixed sample budget}

We first do a fixed sample budget comparison of $\W_2$\texttt{-MC} and $\W_2$\texttt{-Nystr\"{o}m} for completing the Wasserstein distance matrix of $2,000$ points from the OrganCMNIST dataset from the MedMNIST benchmark \cite{yang2023medmnist}, comprising coronal plane CT scans of 11 different organs downsampled to size $28\times28$. For uniform sampling, we subsample the matrix at different rates $\gamma \in\{3\%, 5\%, 10\%, 20\%,25\%\}$. We then match the number of observed entries between uniform sampling and the Nystr\"{o}m method by solving the equation
$ \frac{c(c-1)}{2} + c(n-c) = \gamma \cdot \frac{n(n-1)}{2}$ for $c$ (the number of columns selected by Nystr\"{o}m). This equation equates the number of off-diagonal entries used by both methods. 
For \texttt{MC} in $\W_2$\texttt{-MC}, we use the factorized Gram matrix formulation from \cite{tasissa2018exact} based on the relationship between the Gram matrix and the square distance matrix:
\begin{align}
\mathrm{Find}\quad  &P\in\mathbb{R}^{n\times q}, \nonumber\\
\mathrm{s.t.}\quad  &(PP^\top)_{ii} + (PP^\top)_{jj} - 2(PP^\top)_{ij} = D_{ij}, \nonumber\\
&P^\top 1 = 0, \label{eq:trace objective}
\end{align}
where $q$ denotes the estimate for the rank of the underlying points. The centering constraint can be removed by reparameterization. Defining $\mathcal{A}:\mathbb{R}^{n\times n}\to\mathbb{R}^{|\Omega|}$ as
\[
    \mathcal{A}(X)_{\alpha} = X_{ii} + X_{jj} - 2X_{ij}, \quad \alpha=(i,j)\in\Omega,
\]
the augmented Lagrangian is
\begin{equation}\label{eq:augmented Lagrangian}
\bm{L}(Q;\Lambda) = \tfrac{1}{2}\left\Vert \mathcal{A}(QQ^\top) - b + \Lambda\right\Vert_2^2,
\end{equation}
with $b$ is the vector of observed distances and $\Lambda\in\mathbb{R}^{|\Omega|}$. 
The algorithm starts with an initial estimate of $Q$ (randomly initialized in our experiments) and then alternates between two steps: updating $Q$ via Barzilai–Borwein descent, and updating the Lagrange multipliers. 
For the dataset in consideration, we set the rank estimate to be $200$.

Table \ref{tab:organC_classification} reports the results averaged over 10 random trials. We observe that for fixed sample budget, the Nystr\"{o}m method consistently yields better accuracy compared to matrix completion.

\begin{table}[ht!]
\centering
\begin{tabular}{c|cc}
\toprule
Rate / \# of cols. & $\W_2$\texttt{-Nystr\"{o}m} & $\W_2$\texttt{-MC} \\
\midrule
25\%  / 268 & $~~~1.58\times 10^{-2}$ & $~~~5.20\times 10^{-2}$ \\
 & $\pm1.96\times 10^{-3}$ & $\pm1.91\times 10^{-2}$ \\
20\%  / 211 & $~~~1.85\times 10^{-2}$  & $~~~3.42\times 10^{-2}$ \\
 &  $\pm2.2\times 10^{-3}$  & $\pm8.26\times 10^{-5}$  \\
10\%  / 103 & $~~~2.85\times 10^{-2}$   & $~~~5.65\times 10^{-2}$ \\
 &  $\pm2.60\times 10^{-3}$   & $\pm6.01\times 10^{-3}$ \\
5\%   / 51 & $~~~4.82\times 10^{-2}$  & $~~~1.01\times 10^{-1}$  \\
 & $\pm8.59\times 10^{-3}$ & $\pm2.50\times 10^{-2}$ \\
3\%   / 30  & $~~~7.21\times 10^{-2}$   & $~~~1.32\times 10^{-1}$  \\
  & $\pm1.20\times 10^{-2}$  & $\pm9.77\times 10^{-3}$  \\
\bottomrule
\end{tabular}
\vspace{-0.05in}
\caption{Relative error (mean $\pm$ standard deviation) of Nystr\"{o}m and matrix completion under different sampling rates of a $2000\times2000$ OrganCMNIST $\W_2$ distance matrix. For the Nystr\"{o}m method, the numbers of columns associated with the sample rates are also reported. }\label{tab:organC_classification}
\end{table}

While both Nyström and  typical matrix completion algorithms scale similarly up to logarithmic factors, the Nyström method is a single-step procedure, whereas the cost of the matrix completion algorithm depends on the target accuracy, with the number of iterations required to reach a given error often dominating the runtime.
In our experiments, we run $\W_2$\texttt{-MC} for a sufficient number of iterations, as we consistently observe stable convergence across trials. In practice, choosing a suitable  stopping criterion for matrix completion problems could also be challenging. 
Therefore, this fixed sample budget experiment leads us to conclude that
computing full columns of $\W_2$ distance measurements may be preferred for approximating Wasserstein distance matrix. Studying how far this generalizes is the subject of future work.
\subsection{Classification stability}
We next consider the stability of classification performance under the Nyström approximation as the number of sampled columns varies. Specifically, we compute $5\%,10\%,20\%,\dots,$ $100\%$ of the columns of the $2000 \times 2000$ OrganCMNIST distance matrix using $\W_2$\texttt{-Nystr\"{o}m}. The resulting approximation $D_\est$ is embedded via MDS, followed by a $90-10$ train-test split. For the MDS embedding, we set the dimension to $120$, chosen such that the retained singular values account for $97\%$ of the total sum of singular values.  
Classification is then performed on the embeddings using linear discriminant analysis (LDA), 1-nearest neighbor (KNN1), kernel SVM with Gaussian kernel (SVM\_RBF), and random forest. We note that the dataset has $11$ underlying classes, and we evaluate the quality of the classification using classification accuracy, defined as the proportion of correctly identified data points. 
Figure \ref{fig:organC_classification} shows classification accuracy as a function of the number of columns sampled.

\begin{figure}[!h]t

  \centering
  \vspace{-0.1in}
  \includegraphics[width=.8\linewidth]{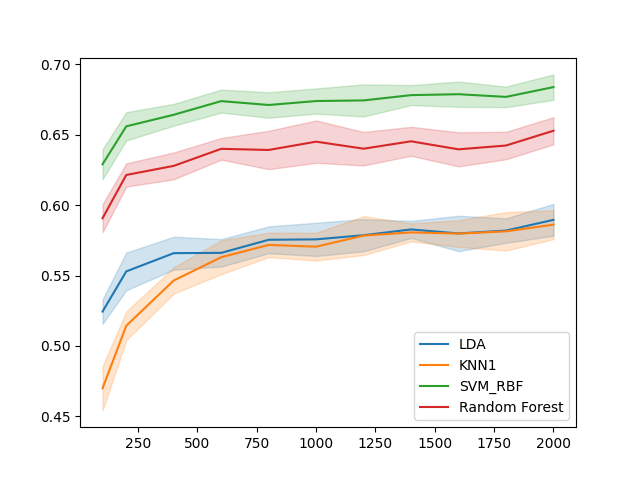}
\vspace{-0.2in}
\caption{Classification accuracy (mean $\pm$ 1 standard deviation of 100 trials) vs. number of columns computed in Algorithm \ref{alg:Nystrom} for 2000 OrganCMNIST data points.} 
\label{fig:organC_classification}
\end{figure}

The main conclusion we draw from Figure \ref{fig:organC_classification} is that the classification performance is quite stable after computing about 10--20\% of the columns of the distance matrix. This result and Theorem \ref{thm:nystrom_approx} suggest that the Nystr\"{o}m Wassmap algorithm (MDS on the output of Algorithm \ref{alg:Nystrom}) is stable after sampling $\mathcal{O}((d+2)\log(d+2))$ columns of $D$.


\section{Conclusion}
We showed that the Nystr\"{o}m method yields a good approximation of $\W_2$ distance matrices, leading to stable MDS embeddings. On the OrganCMNIST dataset,  we showed that classification accuracy  is stable even when only sampling a small fraction of columns, and for a fixed sampling budget, the Nystr\"{o}m method outperforms a distance matrix  completion approach. Unlike classical Euclidean distances, where the tight link between Gram and distance matrices grounds the theory of multidimensional scaling, the structural properties of $\W_2$ distance matrices are far less understood. This gap makes it especially important to develop principled methods that reduce the cost of computing  $\W_2$ distances.
In future work, we will explore theoretical conditions under which $\W_2$ distance matrices admit approximate low-rank structure or fast spectral decay, as well as explore hybrid strategies that combine Nyström-selected columns with partial entry sampling similar following ideas similar to \cite{cai2023matrix}.

\section*{Acknowledgements}
Abiy Tasissa acknowledges partial support from the National Science Foundation through grant
DMS-2208392. HanQin Cai acknowledges partial support from the National Science Foundation through grant
DMS-2304489. 

Muhammad Rana and Keaton Hamm were partially supported by a Research Enhancement Program grant from the College of Science at the University of Texas at Arlington and by the Army Research Office under Grant
Number W911NF-23-1-0213. The views and conclusions contained in this document are those of the authors and
should not be interpreted as representing the official policies, either expressed or implied, of the Army Research
Office or the U.S. Government. The U.S. Government is authorized to reproduce and distribute reprints for
Government purposes notwithstanding any copyright notation herein.





\bibliographystyle{IEEEbib}

\end{document}